\documentclass[twoside,11pt]{article}

\usepackage{jmlr2e}
\usepackage{bbm}
\usepackage{mathrsfs}
\usepackage{amsfonts}
\usepackage{url,subfig,epsfig,graphicx}
\usepackage{amssymb,amsmath}
\usepackage{algorithmic}
\usepackage{algorithm}
\newcommand{\argmin}{\operatornamewithlimits{argmin}}

\ShortHeadings{}{Zonghao}
\firstpageno{1}

\begin{document}

\title{Differentially Private ADMM for Convex Distributed Learning: Improved Accuracy via Multi-Step Approximation}

\author{\name Zonghao Huang\\
       \addr School of Electrical and Computer Engineering\\
       Oklahoma State University\\
       Stillwater, OK 74075, USA\\
       \\
\name Yanmin Gong  \\
       \addr Department of Electrical and Computer Engineering\\
       University of Texas at San Antonio\\
       San Antonio, TX 78249, USA
       }

\editor{Zonghao Huang}

\maketitle

\begin{abstract}
		Alternating Direction Method of Multipliers (ADMM) is a popular algorithm for distributed learning, where a network of nodes collaboratively solve a regularized empirical risk minimization by iterative local computation associated with distributed data and iterate exchanges. When the training data is sensitive, the exchanged iterates will cause serious privacy concern. In this paper, we aim to propose a new differentially private distributed ADMM algorithm with improved accuracy for a wide range of convex learning problems. In our proposed algorithm, we adopt the approximation of the objective function in the local computation to introduce calibrated noise into iterate updates robustly, and allow multiple primal variable updates per node in each iteration. Our theoretical results demonstrate that our approach can obtain higher utility by such multiple approximate updates, and achieve the error bounds asymptotic to the state-of-art ones for differentially private empirical risk minimization.
\end{abstract}

	\section{Introduction} \label{sec:intr}

	
	The advances in machine learning are due to the abundance of data, which could be collected over network but cannot be handled by a single processor. This motivates distributed learning, where data is distributed and possessed by multiple nodes. In distributed learning frameworks, a network of nodes collaboratively solve an optimization problem which is usually formulated as a regularized empirical risk minimization associated with the distributed data. Distributed learning has been widely applied in a variety of areas such as vehicle networks \citep{HaTo17} and wireless sensor networks \citep{PrKu06,GoFa16}.
	
	There exist approaches for distributed optimization including distributed subgradient descent algorithms \citep{NeOl08,NeOz09,LoOz10}, dual averaging methods \citep{DuAg11,TsLa12}, and Alternating Direction Method of Multipliers (ADMM) \citep{BoPa11,LingRi14,ShiLing14,ZhangKw14}. Among these algorithms, ADMM demonstrates fast convergence by both numerical and theoretical results in many applications. Prior works \citep{ShiLing14,MaOz17} have proved that in distributed ADMM, the iterates can converge linearly to the optimal solution while the objective value with feasibility violation can converge to the optimum at a rate of $O(1/t)$, where $t$ is the number of iterations. In this paper, we mainly focus on ADMM-based distributed learning.
	
	In ADMM-based distributed learning, nodes collaboratively solve the regularized empirical risk minimization by iterative local computation and iterate exchanges. The local computation requires each node to solve a local minimization associated with its local dataset while iterate exchanges require nodes to share the updated iterates with their neighbours. When the training data is sensitive, the exchanged learning statistics would cause serious privacy concern \citep{FrJh15,ShSt17}. Therefore, additional privacy-preserving methods are required to control privacy leakage. In this paper, we consider a state-of-art privacy standard, differential privacy \citep{DwRo14,DwMc06,DwKe06}. In ADMM-based distributed learning, differential privacy can be guaranteed by introducing calibrated noise into iterate updates. Recently, there are a few of works \citep{ZhangKh19,ZhZh17,ZhKh18,HuangHu19} focusing on designing differentially private ADMM-based distributed algorithm. Zhang and Zhu \citep{ZhZh17} propose primal variable perturbation and dual variable perturbation to achieve dynamic differential privacy in ADMM. Zhang et al. \citep{ZhKh18} consider adaptive penalty parameters and propose to perturb the penalty. Huang et al. \citep{HuangHu19} propose DP-ADMM by adopting first-order approximation with time-varying Gaussian noise addition, and theoretically demonstrate that their approach can converge at a rate of $O(1/\sqrt{t})$, where $t$ is the number of iterations. Zhang et al. \citep{ZhangKh19} propose Recycled ADMM where the information from odd iterations can be re-utilized in even iterations to save privacy budget. Hu et al. \citep{HuLiu19} propose an ADMM-based algorithm with differential privacy in a distributed data feature setting. However, it is still a challenge to design differentially private ADMM-based distributed algorithms with good privacy-utility trade-off, which has motivated our work.
	%
	%
	
	In this paper, we propose an improved differentially private distributed ADMM algorithm. The key algorithmic features of our approach are to adopt the approximation of the objective function in iterate updates, and to allow multiple primal variable updates per node in each iteration. Our approach adopts the approximation in order to combine the calibrated noise ensuring differential privacy and the ADMM-based learning process robustly, and allows multiple approximate updates per node in each iteration to improve the privacy-utility trade-off. We demonstrate that by adding calibrated noise our approach can achieve differential privacy, and we analyze the utility of our proposed algorithm by the excess empirical risk with feasibility violation. Our theoretical results demonstrate that our approach can obtain higher utility by such multiple iterate updates with approximation, and achieve the error bounds asymptotic to the state-of-art ones for differentially private empirical risk minimization.  

	The main contributions of this paper are summarized as follows:
	\begin{enumerate}
		\item 
		
		
		We propose a new differentially private ADMM-based distributed learning algorithm, where multiple approximate iterate updates with calibrated noise are performed per node in each ADMM iteration to achieve differential privacy and improve the privacy-utility trade-off.
		
		
		\item We analyze the utility of our approach theoretically by the excess empirical risk with feasibility violation. Our theoretical results show that our approach can obtain higher accuracy by multiple approximate iterate updates per node in each iteration and achieve the error bounds asymptotic to the state-of-art ones for differentially private empirical risk minimization.
		
		\item We conduct numerical experiments based on real-world datasets to show the improved privacy-utility trade-off of our approach by comparing with previous works. 
	\end{enumerate}
	
	In the remainder of this paper is organized as follows. In Section \ref{sec:model}, we introduce our problem statement. In Section \ref{sec:appro}, we introduce our proposed algorithm and provide the privacy analysis. In Section \ref{sec:conv}, we give the theoretical utility analysis of our approach. In Section \ref{sec:impl}, we show our numerical results. In Section \ref{sec:rela} and Section \ref{sec:conc}, we introduce the related work and conclude our work.

	\section{Problem Statement} \label{sec:model}
	In this section, we first introduce our problem setting. Then, we describe the ADMM-based distributed learning algorithm, and discuss the associated privacy concern. 
	
	
	\subsection{Problem Setting}
	
	We consider a connected network given by a undirected graph $\mathcal{G}(\mathcal{V},\mathcal{E})$, which consists of a set of nodes $\mathcal{V}=\{1,2,\dots,n\}$, and a set of edges $\mathcal{E}$. In this connected network, each node can only exchange information with its connected neighbours, and we use $\mathcal{N}_i$ to denote the neighbour set of node $i$. Each node $i \in \mathcal{V}$ possesses a private training dataset with size of $m_i$: $\big\{(\boldsymbol{a}_{i,j}, b_{i,j}) ,j \in \mathcal{D}_i\big\}$, where $\boldsymbol{a}_{i,j} \in \mathbb{R}^d$ represents the data feature vector of the $j$-th training sample belonging to $i$, and $b_{i,j} \in \{+1,-1\}$ is the corresponding data label.
	
	
	
	The goal of our problem is to train a supervised learning model on the aggregated dataset $\{\mathcal{D}_i\}_{i \in \mathcal{V}}$, which enables predicting a label for any new data feature vector. The learning objective can be formulated as the following regularized empirical risk minimization problem:
	\begin{align}
	\min_{\boldsymbol{w}} \quad \sum_{i \in \mathcal{V}} \sum_{j\in \mathcal{D}_i}\frac{1}{m_i} \ell(\boldsymbol{a}_{i,j}, b_{i,j},\boldsymbol{w})  +\lambda \phi(\boldsymbol{w}) \label{pro:1},
	\end{align}
	where $\boldsymbol{w} \in \mathcal{W} \subseteq \mathbb{R}^{d }$ is the trained machine learning model, $\ell (\cdot)$ is the loss function used to measure the quality of the trained model, e.g., the loss function $\ell(\boldsymbol{a}, b, \boldsymbol{w})$ can be defined by $\log \big(1+\exp(-b\boldsymbol{w}^{\top}\boldsymbol{a})\big)$ when we consider logistic regression, $\phi(\cdot) $ refers to the regularizer function introduced to prevent overfitting, and $\lambda > 0$ is the regularizer parameter controlling the impact of regularizer.
	
	In this paper, we assume that the loss function $\ell (\cdot)$ and the regularizer function $\phi(\cdot)$ are both convex and Lipschitz. We use $\nabla \ell (\cdot)$ and $\nabla \phi(\cdot)$ to denote their gradient if they are differentiable or subgradient if not differentiable. We use $\Vert \cdot \Vert $ to denote the Euclidean norm.
	
	For completeness, we also give the following additional definitions related to convex functions used in this paper:
		\begin{definition}[Convex Function]
	A function $f(\cdot)$: $\mathcal{C} \rightarrow \mathbb{R}$ is convex if, for all pairs $\boldsymbol{x}, \boldsymbol{y} \in \mathcal{C}$, we have:
	\begin{equation}
	    f(\boldsymbol{x}) \leq f(\boldsymbol{y})+ \langle \nabla f(\boldsymbol{x}), \boldsymbol{x}-\boldsymbol{y} \rangle.
	\end{equation}
	\end{definition}
	\begin{definition}[Lipschitz Function]
A function $f(\cdot)$: $\mathcal{C} \rightarrow \mathbb{R}$ is $c$-Lipschitz if, for all pairs $\boldsymbol{x}, \boldsymbol{y} \in \mathcal{C}$, we have:
	\begin{equation}
	   \vert f(\boldsymbol{x})- f(\boldsymbol{y})\vert \leq c \Vert \boldsymbol{x} - \boldsymbol{y} \Vert.
	\end{equation}
	\end{definition}

	\subsection{ADMM-Based Distributed Learning Algorithm}
	
	To solve problem \eqref{pro:1} with ADMM in distributed manner, we need to reformulate it as:
	\begin{subequations}\label{eq:dispro}
		\begin{align}
		\min_{\{\boldsymbol{w}_i\}}\quad & \sum_{i \in \mathcal{V}} \bigg(\sum_{j \in \mathcal{D}_i}\frac{1}{m_i} \ell(\boldsymbol{a}_{i,j}, b_{i,j},\boldsymbol{w}_i) + \frac{\lambda}{n} \phi(\boldsymbol{w}_i)\bigg), \label{objective} \\
		\text{s.t.} \quad & \boldsymbol{w}_i= \boldsymbol{z}_{i,j},\boldsymbol{z}_{i,j}= \boldsymbol{w}_j,\forall i \in \mathcal{V},  \forall j \in \mathcal{N}_i \label{prob:modified}
		\end{align}
	\end{subequations}
	where $\boldsymbol{w}_i  \in \mathcal{W} \subseteq \mathbb{R}^{d}$ is the local model solved by node $i$, and $\boldsymbol{z}_{i,j} \in \mathbb{R}^{d}$ is an auxiliary variable imposing the consensus constraint on neighboring nodes. The objective function \eqref{objective} is decoupled and constraints \eqref{prob:modified} enforce that all the local models reach consensus finally.
	
	Let $\{\boldsymbol{w}_i\}$, $\{\boldsymbol{z}_{i,j} \}$, and $\{\boldsymbol{\gamma}^a_{i,j}\}$ be the shorthand for $\{\boldsymbol{w}_i\}_{i \in \mathcal{V}}$, \\$\{\boldsymbol{z}_{i,j} \}_{i \in \mathcal{V}, j \in \mathcal{N}_i}$, and $\{\boldsymbol{\gamma}^a_{i,j}\}_{i \in \mathcal{V}, j \in \mathcal{N}_i, a= 1, 2}$, respectively. We define:
		\begin{equation}
L_{\mathcal{D}_i}(\boldsymbol{w}_i)  
	= \sum_{j \in \mathcal{D}_i}\frac{1}{m_i} \ell(\boldsymbol{a}_{i,j}, b_{i,j},\boldsymbol{w}_i) + \frac{\lambda}{n} \phi(\boldsymbol{w}_i).
	\end{equation}
	The augmented Lagrangian function associated with the problem \eqref{eq:dispro} is:
	\begin{equation}
	\begin{split}
	\mathcal{L}\big(\{\boldsymbol{w}_i\}, &\{\boldsymbol{z}_{i,j} \}, \{\boldsymbol{\gamma}^a_{i,j}\} \big)
	=  \sum_{i \in \mathcal{V}} \bigg(L_{\mathcal{D}_i}(\boldsymbol{w}_i)  - \sum_{j \in \mathcal{N}_i}\big\langle \boldsymbol{\gamma}_{i,j}^1, \boldsymbol{w}_i - \boldsymbol{z}_{i,j} \big\rangle \\& - \sum_{j \in \mathcal{N}_i}\big\langle \boldsymbol{\gamma}_{i,j}^2,   \boldsymbol{z}_{i,j} -\boldsymbol{w}_j \big\rangle  + \frac{\rho}{2}  \sum_{j \in \mathcal{N}_i}\big( {\Vert \boldsymbol{w}_i -  \boldsymbol{z}_{i,j}\Vert}^2+  {\Vert \boldsymbol{w}_j -  \boldsymbol{z}_{i,j}\Vert}^2\big)\bigg),  \label{eq:lag}   
	\end{split}
	\end{equation}
	where $\{\boldsymbol{\gamma}^a_{i,j}\}$ are the dual variables associated with constraints~\eqref{prob:modified} and $\rho > 0$ is the penalty parameter. The ADMM solves the problem \eqref{eq:dispro} in a Gauss-Seidel manner by minimizing \eqref{eq:lag} w.r.t. $\{\boldsymbol{w}_i\}$ and $\{\boldsymbol{z}_{i,j} \}$ alternatively followed by dual updates of $\{\boldsymbol{\gamma}^a_{i,j}\}$:
	\begin{subequations} \label{update:1}
		\begin{align}
		&\boldsymbol{w}_i^{k} =  \argmin_{\boldsymbol{w}_i} \mathcal{L}\big(\{\boldsymbol{w}_i\}, \{\boldsymbol{z}^{k-1}_{i,j} \}, \{\boldsymbol{\gamma}_{i,j}^{a,k-1}\}\big), \\
		&\boldsymbol{z}_{i,j}^{k} =   \argmin_{\boldsymbol{z}_{i,j}}\mathcal{L}\big(\{\boldsymbol{w}_i^{k}\}, \{ \boldsymbol{z}_{i,j} \} , \{\boldsymbol{\gamma}_{i,j}^{a,k-1}\}  \big), \\
		& \boldsymbol{\gamma}_{i,j}^{1,k} =  \boldsymbol{\gamma}_{i,j}^{1,k-1} - \rho (\boldsymbol{w}_i^{k} - \boldsymbol{z}_{i,j}^{k}),\\
		& \boldsymbol{\gamma}_{i,j}^{2,k} =  \boldsymbol{\gamma}_{i,j}^{2,k-1} - \rho ( \boldsymbol{z}_{i,j}^{k}-\boldsymbol{w}_j^{k}).
		\end{align}
	\end{subequations}
	
	According to the previous works \citep{FoCa10}, the above iterate updates could be simplified by initializing $\boldsymbol{\gamma}^{1,0}_{i,j} = \boldsymbol{\gamma}_{i,j}^{2,0} = \boldsymbol{0}$, which can enforce $\boldsymbol{\gamma}^{1,k}_{i,j} = \boldsymbol{\gamma}_{i,j}^{2,k}$ and $\boldsymbol{z}^{k}_{i,j} = \frac{1}{2}(\boldsymbol{w}_i^k+\boldsymbol{w}_j^k)$. Let $\boldsymbol{\gamma}_i^k = \sum_{j \in \mathcal{N}_i} \boldsymbol{\gamma}_{i,j}^{1,k} =\sum_{j \in \mathcal{N}_i} \boldsymbol{\gamma}_{i,j}^{2,k}$, let $\{\boldsymbol{w}_j^{k}\}$ be the shorthand of $\{\boldsymbol{w}_j^{k}\}_{j \in \mathcal{N}_i}$, and define $\mathcal{L}^s_i\big(\boldsymbol{w}_i,\boldsymbol{w}_i^{k},\{\boldsymbol{w}_j^{k}\},  \boldsymbol{\gamma}_{i}^{k} \big)$ as:
	\begin{equation}
	\begin{split}
	\mathcal{L}^s_i\big(\boldsymbol{w}_i,\boldsymbol{w}_i^{k},\{\boldsymbol{w}_j^{k}\},  \boldsymbol{\gamma}_{i}^{k}\big)=  L_{\mathcal{D}_i}(\boldsymbol{w}_i)- 2\big\langle \boldsymbol{\gamma}_{i}^{k}, \boldsymbol{w}_i \big\rangle   + \rho \sum_{j \in \mathcal{N}_i} {\Vert \boldsymbol{w}_i -\frac{1}{2}(\boldsymbol{w}_{i}^{k}+ \boldsymbol{w}_{j}^{k}) \Vert}^2. \label{eq:sim}
	\end{split}
	\end{equation}
	Iterate updates \eqref{update:1} can be simplified as:
	\begin{subequations}
		\begin{align}
		\boldsymbol{w}_i^{k} = & \argmin_{\boldsymbol{w}_i} \mathcal{L}^s_i\big(\boldsymbol{w}_i,\boldsymbol{w}_i^{k-1},\{\boldsymbol{w}_j^{k-1}\},  \boldsymbol{\gamma}_{i}^{k-1} \big), \label{up:1}\\
		\boldsymbol{\gamma}_i^{k} =& \boldsymbol{\gamma}_i^{k-1} - \frac{\rho}{2} \sum_{j \in \mathcal{N}_i} \big(\boldsymbol{w}_i^{k} - \boldsymbol{w}_j^{k} \big), \label{up:12}
		\end{align}
	\end{subequations}
	where Eq. \eqref{up:1} is regarded as the primal varibale update while Eq. \eqref{up:12} is known as the dual variable update.
	
	

	\subsection{Privacy Concern}
	
	In ADMM, iterates are updated by solving a minimization associated with the local dataset (Eq. \eqref{up:1}), and they are needed to shared with neighbours. According to the previous works \citep{FrJh15,ShSt17}, adversary can infer the data information from the released learning statistics. If the local training data is sensitive, the shared iterates would cause privacy leakage.  
	
	The main goal of this paper is to provide privacy protection in ADMM against inference attacks from an adversary, who tries to infer sensitive information about the nodes' private datasets from the shared messages.
	
	
	In order to provide privacy guarantee against such attacks, we define our privacy model formally by the notion of differential privacy \citep{DwMc06,DwRo14}. 
	Specifically, we adopt the $(\epsilon,\delta)$-differential privacy defined as follows:
	
	\begin{definition}[$(\epsilon,\delta)$-Differential Privacy] \label{de:1}
		A randomized mechanism $\mathcal{M}$ is $(\epsilon,\delta)$-differentially private if for any two neighbouring datasets $\mathcal{D}$ and $\mathcal{D}^{'}$ differing in only one tuple, and for any output subset $\mathcal{O} \subseteq$ range($\mathcal{M}$):
		\begin{equation}
		\Pr\big[\mathcal{M}(\mathcal{D}) \in \mathcal{O}\big] \leq e^{\epsilon} \cdot \Pr\big[\mathcal{M}(\mathcal{D}^{'}) \in \mathcal{O}\big] + \delta, 
		\end{equation}
		which means, with probability of at least $1-\delta$, the ratio of the probability distributions for two neighboring datasets is bounded by $e^{\epsilon}$.
	\end{definition}
	
	In Definition~\ref{de:1}, $\delta$ and $\epsilon$ indicate the strength of privacy protection from the mechanism (a smaller $\epsilon$ or a smaller $\delta$ gives better privacy protection). Gaussian mechanism is a widely used randomization method used to guarantee $(\epsilon,\delta)$-differential privacy, where calibrated noise sampled from normal (Gaussian) distribution is added to the output. 
	
	When we consider a class of differetially private algorithms under $t$-fold adaptive composition where the auxiliary inputs of the $k$-th algorithm are the outputs of all previous algorithms, we use the following moments accoutant-based advanced composition theorem to analyze the privacy guarantee.
	\begin{theorem}[Advanced Composition] \label{the:com}
		Let $\epsilon, \delta \geq 0$. The class of $(\epsilon,\delta)$-differentially private algorithms satisfies \\$(\epsilon^{'},\delta)$-differential privacy under $t$-fold adaptive composition, where $\epsilon^{'} = c_0 \sqrt{t}\epsilon$ for some constant $c_0$. 
	\end{theorem}
	\begin{proof}
		The proof of the advanced composition theorem is based on the moments accountant method proposed in \citep{AbCh16}. In moments accountant method, the $\tau$-th log moments of privacy loss from each $(\epsilon,\delta)$-differentially private algorithm can be given by $\frac{t\tau(\tau+1)\epsilon^2}{4\ln(1.25/\delta)}$. According to the linear composability of the log moments, we obtain the log moment of the total privacy loss from the class of private algorithms by $\frac{t\tau(\tau+1)\epsilon^2}{4\ln(1.25/\delta)}$. By using the tail bound property of the log moment, we can obtain the relationship between $\epsilon$ and $\epsilon^{'}$, which is $ \epsilon^{'} \geq \sqrt{\frac{t\ln(1/\delta)}{\ln(1.25/\delta)}}\epsilon$. Thus, there exists a constant $c_0$ so that  $\epsilon^{'} = c_0 \sqrt{t}\epsilon$. Due to the limited space here, we suggest readers to refer to the previous works \citep{AbCh16,HuangHu19} for the details.
	\end{proof}
	
	\section{Improved Differentially Private Distributed ADMM} \label{sec:appro}
	\subsection{Main Idea}
	
	As discussed in the last section, the privacy concern in ADMM-based distributed learning comes from the exchanged iterates. Calibrated noise is added to the iterates to control the privacy leakage and guarantee differential privacy. In order to introduce the calibrated noise into ADMM robustly, we adopt the approximation of the objective function when updating the primal variables. Such approximation is used in linearized ADMM \citep{LingRi14} to reduce computation cost, stochasitc ADMM \citep{OuHe13}, and some previous works on differentially private ADMM \citep{HuangHu19,ZhangKh19}. In addition, inspired by the federated learning framework proposed by \citep{McMo16}, our approach allows to perform multiple primal variable updates based on the approximate function per node in each iteration, in order to improve the accuracy.
	
	
	
	
	\subsection{Our Approach}

	Our proposed algorithm adopts the approximation of the objective when updating the primal variable, and allows performing $l$ updates with calibrated noise per node in each iteration. Here we use $\tilde{\boldsymbol{w}}_i^{k,r}$ to denote the $r$-th noisy primal variable from node $i$ in the $k$-th ADMM iteration.
	
	In each iteration, each node does not perform the exact minimization to update the primal variable by \eqref{up:1}. Instead, each node performs the inexact minimization by adopting the approximation of the objective at $\tilde{\boldsymbol{w}}_i^{k,r}$:
	\begin{equation}
	    L_{\mathcal{D}_i}(\boldsymbol{w}_i) \approx L_{\mathcal{D}_i}(\tilde{\boldsymbol{w}}_i^{k,r}) + \langle \nabla L_{\mathcal{D}_i}(\tilde{\boldsymbol{w}}_i^{k,r}), \tilde{\boldsymbol{w}}_i^{k,r} -\boldsymbol{w}_i \rangle + \frac{\eta_i^{k,r+1}}{2} \Vert \boldsymbol{w}_i - \tilde{\boldsymbol{w}}_i^{k,r}\Vert^2 , \label{eq:appr}
	\end{equation}
	where $\eta_{i}^{k,r+1}$ is an approximation parameter to control the distance between the updated variable and the previous one. By approximation in Eq. \eqref{eq:sim}, we define $\hat{\mathcal{L}}_i^s \big(\boldsymbol{w}_i,\tilde{\boldsymbol{w}}_i^{k,r},\tilde{\boldsymbol{w}}_i^{k-1},\{\tilde{\boldsymbol{w}}_j^{k-1}\}, \boldsymbol{\gamma}_i^{k-1} \big)$ by:
	\begin{equation}
	\begin{split}
	\hat{\mathcal{L}}_i^s \big(\boldsymbol{w}_i,&\tilde{\boldsymbol{w}}_i^{k,r},\tilde{\boldsymbol{w}}_i^{k-1}, \{\tilde{\boldsymbol{w}}_j^{k-1}\}, \boldsymbol{\gamma}_i^{k-1} \big)  =  L_{\mathcal{D}_i}(\tilde{\boldsymbol{w}}_i^{k,r})   + \langle \nabla L_{\mathcal{D}_i}(\tilde{\boldsymbol{w}}_i^{k,r}), \tilde{\boldsymbol{w}}_i^{k,r} -\boldsymbol{w}_i \rangle \\ & + \frac{\eta_i^{k,r+1}}{2} \Vert \boldsymbol{w}_i - \tilde{\boldsymbol{w}}_i^{k,r}\Vert^2 
	- 2\big\langle \boldsymbol{\gamma}_{i}^{k-1}, \boldsymbol{w}_i \big\rangle  +  \rho \sum_{j \in \mathcal{N}_i} {\Vert \boldsymbol{w}_i -\frac{1}{2}(\tilde{\boldsymbol{w}}_{i}^{k-1} + \tilde{\boldsymbol{w}}_{j}^{k-1}) \Vert}^2 . \label{eq:app}
	\end{split}
	\end{equation}
	Our approach allows $l$ primal variable updates per node in each iteration based on the approximate function. Thus, step \eqref{up:1} is replaced by an inner $l$-iterative process:
	\begin{subequations}
		\begin{align}
		\boldsymbol{w}_i^{k,r+1} = & \min_{\boldsymbol{w}_i} \hat{\mathcal{L}}_i^s \big(\boldsymbol{w}_i,\tilde{\boldsymbol{w}}_i^{k,r},\tilde{\boldsymbol{w}}_i^{k-1},\{\tilde{\boldsymbol{w}}_j^{k-1}\}, \boldsymbol{\gamma}_i^{k-1} \big) , \label{up:3}\\
		\tilde{\boldsymbol{w}}^{k,r+1} =  & \boldsymbol{w}_i^{k,r+1} + \boldsymbol{\xi}^{k,r+1}_i, \label{up:4}
		\end{align}
	\end{subequations}
	where $\boldsymbol{\xi}^{k,r+1}_i$ is the sampled noise from normal (Gaussian) distribution to ensure differential privacy:
	\begin{equation}
	\boldsymbol{\xi}_i^{k,r+1} \sim \mathcal{N}(0, s_i^{{k,r+1}^2} \sigma^2 \mathrm{I}_d ).
	\end{equation}
	After the $l$-iterative primal variable updates, the dual variable update follows as:
	\begin{equation}
	\boldsymbol{\gamma}_i^{k} =\boldsymbol{\gamma}_i^{k-1} - \frac{\rho}{2} \sum_{j \in \mathcal{N}_i} \big(\tilde{\boldsymbol{w}}_i^{k} -\tilde{\boldsymbol{w}}_j^{k}\big) , \label{up:5}
	\end{equation}
	where $\tilde{\boldsymbol{w}}_i^{k}  = \frac{1}{l} \sum_{r=1}^{l} \tilde{\boldsymbol{w}}_i^{k,r} $.
	
	The details of our approach are given in Algorithm~\ref{ag:3}. Each node $i$ firstly initializes its noisy primal variables $ \boldsymbol{\tilde{w}}_i^{0,l} $ and $\boldsymbol{\tilde{w}}_i^{0}$, and dual variables $ \boldsymbol{\gamma}_i^0 $. Then each node $i$ updates its noisy primal variables by an inner $l$-iterative process, where $\boldsymbol{\tilde{w}}_i^{k,r+1}$ is updated by \eqref{up:3} and \eqref{up:4} in each inner iteration. After $l$ iterations of the inner process, node $i$ obtain a noisy primal variable $\boldsymbol{\tilde{w}}_i^{k,l}$ and $\tilde{\boldsymbol{w}}_i^{k} $, and broadcast $\tilde{\boldsymbol{w}}_i^{k} $ to its neighbours $j \in \mathcal{N}_i$. After receiving the noisy primal variables $\{\boldsymbol{\tilde{w}}_j^{k}\}_{j \in \mathcal{N}_i}$ from its neighbours, node $i$ continues to update its dual variable $\boldsymbol{\gamma}_i^{k}$ by \eqref{up:5}. The iterative process will continue until reaching $t$ iterations.
	
		\begin{algorithm}
		\caption{Improved Differentially Private Distributed ADMM}\label{ag:3}
		\begin{algorithmic}[1]
			\STATE Initialize  $\{\tilde{\boldsymbol{w}}_i^{0,l}\}_{i \in \mathcal{V}}$, $\{\tilde{\boldsymbol{w}}_i^{0}\}_{i \in \mathcal{V}}$ and $\{\boldsymbol{\gamma}_i^0\}_{i \in \mathcal{V}}$;
			
			\FOR{$k =  1, 2, \dots, t$}
			\FOR{$i \in \mathcal{V}$}
			\STATE Let $\tilde{\boldsymbol{w}}_i^{k,0} = \tilde{\boldsymbol{w}}_i^{k-1,l}$.
			\FOR{$r = 0 , 1, \dots, l-1$}
			\STATE Sample $\boldsymbol{\xi}_i^{k,r+1} \sim \mathcal{N}(0, s_i^{{k,r+1}^2} \sigma^2 \mathrm{I}_d )$;
			\STATE Compute $ \boldsymbol{w}_i^{k,r+1}$ by Eq. \eqref{up:3};
			\STATE  Compute $  \tilde{\boldsymbol{w}}_i^{k,r+1}$ by Eq. \eqref{up:4};
			\ENDFOR
			\STATE Compute $\tilde{\boldsymbol{w}}_i^{k}  = \frac{1}{l}\sum_{r=1}^{l} \tilde{\boldsymbol{w}}_i^{k,r}$.
			\ENDFOR
			\FOR{$i \in \mathcal{V}$}
			\STATE Broadcast $\tilde{\boldsymbol{w}}_i^{k}$ to all neighbours $j \in \mathcal{N}_i$;
			\ENDFOR
			\FOR{$i \in \mathcal{V}$}
			\STATE  Compute $  \boldsymbol{\gamma}_i^{k}$ by Eq. \eqref{up:5}. 
			\ENDFOR
			\ENDFOR
		\end{algorithmic}
	\end{algorithm}
	
	\textit{Note 1}: In Algorithm \ref{ag:3}, The inner iterative process leads to higher computation cost with a choice of larger $l$. In order to release the computation burden, we can use the stochastic variant of $\nabla L_{\mathcal{D}_i}(\tilde{\boldsymbol{w}}_i^{k,r})$ when updating the primal variable in Step $7$ by sampling a batch of data with replacement. This stochastic variant also leads to the same utility guarantee in Theorem \ref{the:li}.
	
	\textit{Note 2}: In Algorithm \ref{ag:3}, each node shares the averaged result $\tilde{\boldsymbol{w}}_i^{k}$ instead of the latest one $\tilde{\boldsymbol{w}}_i^{k,l}$. And the dual variable is updated based on the shared parameters $\{\boldsymbol{\tilde{w}}_j^{k}\}_{j \in \mathcal{N}_i}$.
	

	
	
	\subsection{Privacy Analysis}
	In this section, we define the $l_2$ norm sensitivity $s^{k,r+1}_i$ and noise magnitude $\sigma$ to achieve $(\epsilon,\delta)$-differential privacy in Algorithm \ref{ag:3}.
	
	\begin{lemma}[$L_2$-Norm Sensitivity]
		Assume that the loss function $\ell(\cdot)$ is $c_1$-Lipschitz. The $l_2$ norm sensitivity of the primal variable update function (Eq. \eqref{up:3}) is given by:
		\begin{equation}
		s^{k,r+1}_i = \frac{2c_1}{(2\rho\vert \mathcal{N}_i \vert+\eta_i^{k,r+1})m_i}.
		\end{equation}
	\end{lemma}
	\begin{proof}
		The $l_2$ norm sensitivity of the primal variable update function (Eq. \eqref{up:3}) is defined by:
		\begin{equation}
		s^{k,r+1}_i  = \max_{\mathcal{D}_i,\mathcal{D}_i^{'}}\big\Vert \boldsymbol{w}_{i,\mathcal{D}_i}^{k,r+1} - \boldsymbol{w}_{i,\mathcal{D}_i^{'}}^{k,r+1}  \big\Vert.
		\end{equation}
		According to Eq. \eqref{eq:app} and Eq. \eqref{up:3}, we obtain a closed-form solution to $\boldsymbol{w}_i^{k,r+1}$:
		\begin{equation}
		\begin{split}
		\boldsymbol{w}_i^{k,r+1} =  \big( - L_{\mathcal{D}_i}(\tilde{\boldsymbol{w}}_i^{k,r})  + 2\boldsymbol{\gamma}_i^{k-1} +\rho\sum_{j \in \mathcal{N}_i}\tilde{\boldsymbol{w}}_j^{k-1}    +\rho\vert \mathcal{N}_i\vert\tilde{\boldsymbol{w}}_i^{k-1}  +\eta_i^{k,r+1}\tilde{\boldsymbol{w}}_i^{k,r} \big)/(2\rho\vert \mathcal{N}_i \vert+ \eta_i^{k,r}).
		\end{split}
		\end{equation}
		Then, we have:
		\begin{equation}
		\begin{split}
		 \big\Vert \boldsymbol{w}_{i,\mathcal{D}_i}^{k,r+1} - \boldsymbol{w}_{i,\mathcal{D}_i^{'}}^{k,r+1}  \big\Vert 
		= & \frac{ \big\Vert \nabla \ell(\boldsymbol{a}_{i,j},b_{i,j},\tilde{\boldsymbol{w}}_{i}^{k,r})-\nabla \ell(\boldsymbol{a}^{'}_{i,j},b^{'}_{i,j},\tilde{\boldsymbol{w}}_{i}^{k,r}) \big\Vert}{(2\rho\vert \mathcal{N}_i \vert+\eta_i^{k,r+1}) m_i}
		\\ \leq & \frac{2\big\Vert \nabla \ell(\cdot) \big\Vert}{(2\rho\vert \mathcal{N}_i \vert+\eta_i^{k,r+1})m_i}.
		\end{split}
		\end{equation}
		Since function $\ell(\cdot)$ is $c_1$-Lipschitz, we obtain the result: $s^{k,r}_i = \frac{2c_1}{(2\rho\vert \mathcal{N}_i \vert+\eta_i^{k,r})m_i}$.
	\end{proof}
	
	\begin{theorem}[Privacy Guarantee]
		Let $\epsilon, \delta \geq 0$ be arbitrary. There exists constant $c_0$ so that Algorithm \ref{ag:3} achieves $(\epsilon,\delta)$-differential privacy if we set the noise magnitude $\sigma$ in Gaussian distribution $\mathcal{N}(0 , s^{{k,r+1}^2}_i \sigma^2 \mathrm{I}_d)$ by:
		\begin{equation}
		\sigma = \frac{c_0  \sqrt{t \cdot l \cdot  2\ln(1.25/\delta)} }{ \epsilon }. 
		\end{equation}
	\end{theorem}
	\begin{proof}
		Due to the limited space, we only provide the proof sketch here. We first follow Theorem A.1. in \citep{DwRo14} to demonstrate that by setting $\sigma = \frac{c_0  \sqrt{t \cdot l \cdot  2\ln(1.25/\delta)} }{ \epsilon }$, each primal variable update function with Gaussian noise satisfies $(\frac{1}{c_0\sqrt{t \cdot l}} \epsilon , \delta)$-differential privacy. Since our approach includes a class of $t\times l$ differentially private gradient functions. By adopting the advanced composition (Theorem \ref{the:com}), we prove that Algorithm \ref{ag:3} achieves $(\epsilon,\delta)$-differential privacy. 
	\end{proof}
	
	\section{Utility Analysis}   \label{sec:conv}
	
	In this section, we theoretically analyze the utility of Algorithm \ref{ag:3}, which can be measured by the expected excess empirical risk with feasibility violation, namely:
	\begin{equation}
	\mathbb{E}\big[\sum_{i \in \mathcal{V}} L_{\mathcal{D}_i}(\hat{\boldsymbol{w}}_i)-L_{\mathcal{D}_i}(\boldsymbol{w}^{*})\big]+ \beta \sum_{i \in \mathcal{V}}\sum_{j \in \mathcal{N}_i} \Vert \hat{\boldsymbol{w}}_i - \hat{\boldsymbol{w}}_j \Vert,
	\end{equation}
	where $\{\hat{\boldsymbol{w}}_i\}_{i \in \mathcal{V}}$ and $\{\hat{\boldsymbol{w}}_i\}_{ i \in \mathcal{V}}$ are the outputs of our proposed algorithm, and $\boldsymbol{w}^{*}$ is the true minimizer of problem \eqref{pro:1}. The excess empirical risk measures the accuracy of the trained model by our approach while the feasibility violation measures the differences of local models.
	
	\subsection{Main Results}
	
	We analyze the excess empirical risk of our approach under the assumption: the objective function is convex and Lipschitz. 
	
	\begin{theorem}[Utility Analysis] \label{the:li}
		Assume the objective function $L(\cdot)$ is $c_2$-Lipschitz, the diameter of the $\mathcal{W}$ is bounded by $D$, namely $\sup_{\boldsymbol{w},\boldsymbol{w}^{'} \in \mathcal{W}}\Vert \boldsymbol{w}-\boldsymbol{w}^{'} \Vert \leq D$, and the domain of dual variable is bounded, namely $\Vert \boldsymbol{\gamma}_{i,j} \Vert \leq \beta$ Let
		\begin{equation}
		\hat{\boldsymbol{w}}_i = \frac{1}{t}\frac{1}{l}\sum_{k=1}^{t} \sum_{r=0}^{l-1} \tilde{\boldsymbol{w}}_i^{k,r} .
		\end{equation}
		If we set the learning rate:
		\begin{equation}
		\eta_i^{k,r} = \frac{\sqrt{2 k r}}{D}\sqrt{\frac{c_2^2}{n^2}+ \frac{d c_0^2 c_1^2 t l 8\ln{(1.25/\delta)}}{\epsilon^2 m_i^2}} ,
		\end{equation}
		we have the following expected error bound:
		\begin{equation}
		\begin{split}
		& \mathbb{E}\bigg[\sum_{i \in \mathcal{V}}L_{\mathcal{D}_i}(\hat{\boldsymbol{w}}_i)-L_{\mathcal{D}_i}(\boldsymbol{w}^{*}) \bigg] + \beta \sum_{i \in \mathcal{V}}\sum_{j \in \mathcal{N}_i} \Vert \hat{\boldsymbol{w}}_i - \hat{\boldsymbol{w}}_j \Vert \\
		\leq & \sum_{i \in \mathcal{V }}\bigg(\frac{\sqrt{2}D }{\sqrt{t \cdot l}}\big(\frac{c_2^2}{n^2}+ \frac{d c_0^2 c_1^2 t l 8\ln{(1.25/\delta)}}{\epsilon^2 m_i^2}\big)^{\frac{1}{2}}  + \frac{\rho \vert \mathcal{N}_i\vert D^2+\vert \mathcal{N}_i\vert\beta^2/\rho}{t}\bigg).
		\end{split}
		\end{equation}
	\end{theorem}
	\begin{proof}
		See the proof in Appendix.
	\end{proof}

	\subsection{Discussion} \label{sec:com}
	
	Our main result (Theorem \ref{the:li}) demonstrates the privacy-utility trade-off of our approach. When we want to ensure smaller privacy leakage from our output by setting a smaller $\epsilon$, the utility of our approach will decrease. When $\epsilon$ is set to be extremely large leading to no privacy protection, our approach achieves a sublinear convergence rate of $O(1/t)$, which matches the result from the conventional ADMM method.
	
	Theorem \ref{the:li} demonstrates that our approach can be improved by allowing multiple approximate primal variable updates. When we set a larger $l$ in our proposed algorithm, our approach can achieves better accuracy but this also leads to higher computation cost in each iteration.
	
	From Theorem \ref{the:li}, if the iteration number $t$ is sufficiently large enough our approach achieve the error bound $O\bigg(\sum_{i \in \mathcal{V}}\frac{\sqrt{d\ln(1/\delta)} }{ m_i\epsilon}\bigg)$, which is comparable to the state-of-art error bound  $O\big( \frac{\sqrt{d\ln{(1/\delta)}}}{N\epsilon} \big)$ for differentially private empirical risk minimization under the assumption that the objective is Lischipz and convex, here $N$ is the total number of training data.

	\section{Numerical Results}   \label{sec:impl}
	
	\begin{figure}[t]
		\centering
		\subfloat{\includegraphics[width=3in,height=2in]{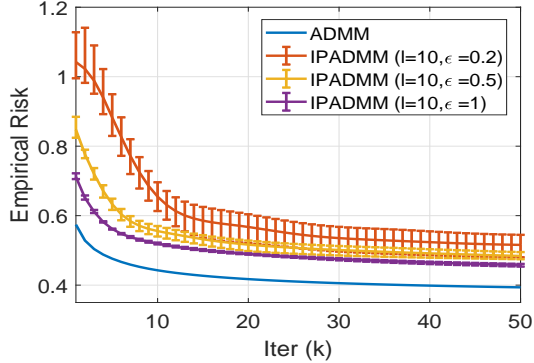}%
		}
		\hfil
		\caption{Utility and Privacy Trade-off.}
		\label{fig:1}
		\vspace*{-.15in}
	\end{figure}
	
	\begin{figure}[t]
		\centering
		\subfloat{\includegraphics[width=3in,height=2in]{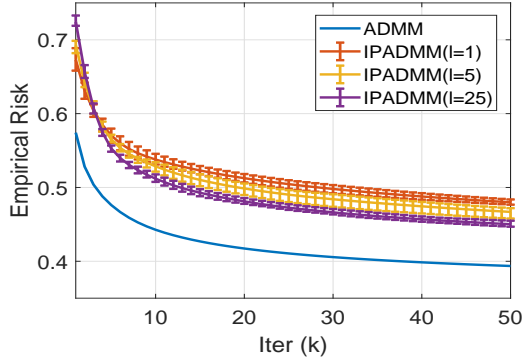}%
		}
		\hfil
		\caption{Impact of $l$ on Accuracy.}
		\label{fig:2}
		\vspace*{-.15in}
	\end{figure}

	\begin{figure}[t]
		\centering
		\subfloat[$\epsilon = 0.5$, $\delta = 10^{-5}$]{\includegraphics[width=3in,height=2in]{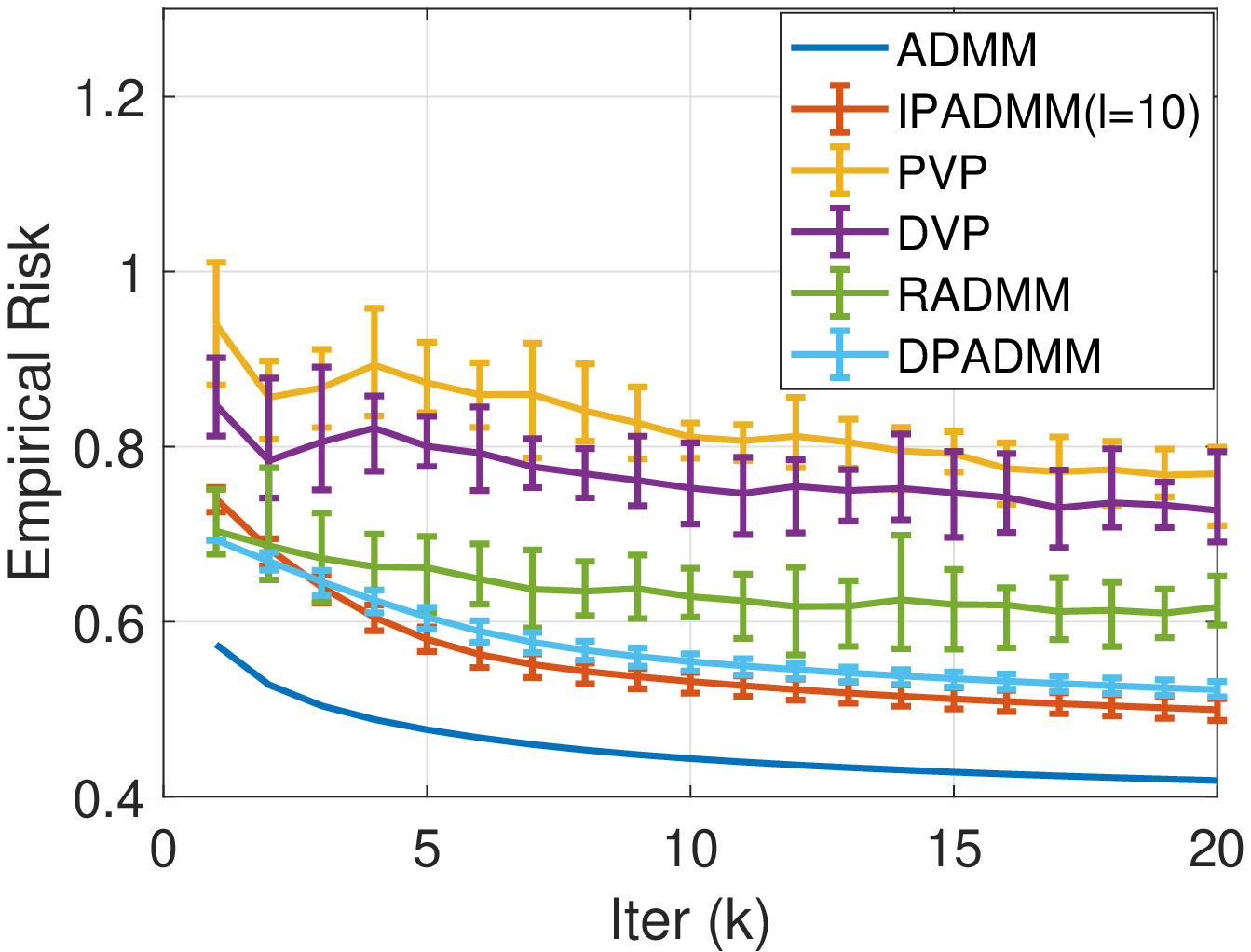}%
		}
		\hfil
		\subfloat[$\epsilon = 1$, $\delta = 10^{-5}$]{\includegraphics[width=3in,height=2in]{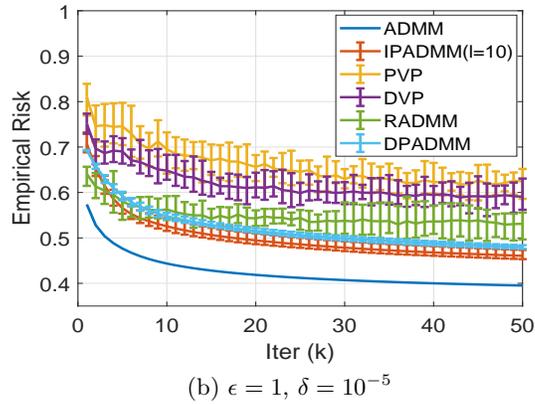}%
		}
		\hfil
		\caption{Accuracy Comparison in Empirical Risk.}
		\label{fig:3}
		\vspace*{-.15in}
	\end{figure}

	In this section, we give the numerical results of our approach and compare our proposed algorithm (Algorithm \ref{ag:3}) with algorithms proposed by prior works by the simulation in MATLAB.
	
	\subsection{Regularized Logistic Regression}
	
	We evaluate our approach by $l_2$ regularized logistic regression. Logistic regression is a widely used statistical model for classification, and its loss function is described as:
	\begin{equation}
	\ell (\boldsymbol{a},b,\boldsymbol{w}) = \log \big(1+\exp(-b\boldsymbol{w}^{\top}\boldsymbol{a})\big).
	\end{equation}
	Thus by $l_2$ regularized logistic regression, the objective of our regularized empirical risk minimization problem can be formulated as:
	\begin{equation}
	\begin{split}
	 L_{\mathcal{D}_i}(\boldsymbol{w}_i)    = \sum_{j \in \mathcal{D}_i}\frac{1}{m_i} \log \big(1+\exp(-b_{i,j}\boldsymbol{w}_i^{\top}\boldsymbol{a}_{i,j})\big) + \frac{\lambda}{2n} \Vert\boldsymbol{w}_i\Vert^2.
	\end{split}
	\end{equation}
	
	\subsection{Dataset}
	
	The dataset used in our simulation is Adult dataset \citep{data:Ad} from UCI Machine Learning Repository. Adult dataset includes $48,842$ instances, each of which has $14$ personal attributes with a label representing whether the income is above $\$ 50,000$ or not. We follow the previous works \citep{ZhKh18,HuangHu19} to preprocess the data by removing all the instances with missing values, converting the categorical attributes into binary vectors, normalizing columns to guarantee the maximum value of each column is 1, normalizing rows to enforce their $l_2$ norm to be less than 1, and converting the labels $\{>50k, <50k\}$ into $\{+1,-1\}$. After the data preprocessing, we obtain $45,222$ data instances each with a $104$-dimensional feature vector and a label belonging to $\{+1,-1\}$.  
	
	\subsection{Baseline Algorithms}	
	We compare our approach: Improved differentially Private Distributed ADMM (IPADMM) with four baseline algorithms: (1) non-private distributed ADMM algorithm, (2) ADMM algorithm with PVP in \citep{ZhZh17},
	(3) ADMM with dual variable perturbation (DVP) in \citep{ZhZh17}, (4) Recycled ADMM (RADMM) proposed in \citep{ZhangKh19}, and (5) DPADMM proposed in \citep{HuangHu19}.
	
	\subsection{Simulation Results}
	
	We mainly evaluate our approacch on the utility-privacy trade-off and the effect of the choice of $l$ on utility, and compare our approach with the baseline algorithms. In our simulation, we set $\rho$ to be $0.001$ and $\mu$ to be $0.0001$, and by data preprocessing, we can enforce the objective function to be $(n+\frac{\mu D}{n})$-Lischiptz, where $n$ is the number of nodes and $D$ is the diameter of $\mathcal{W}$. Here, we consider a network consisting of $100$ nodes fully connected and with evenly divided data. Our numerical results are the averaged results from $10$ simulations. 
	
	Figure \ref{fig:1} shows the utility-privacy trade-off of our approach. Here we fix $\delta$ to be $10^{-5}$ and $l$ to be $10$, and only change $\epsilon$. With increasing $\epsilon$ indicating weaker privacy guarantee, our approach has less empirical risk achieving better utility, which is consistent with Theorem \ref{the:li}. 
	
	Figure \ref{fig:2} demonstrates the impact of the choice of $l$ on the utility of our approach. In this simulation, we fix $\epsilon$ to be $1$ and change $l$ from $1$ to $25$. When we set a larger $l$, the accuracy of our algorithm can be improved, which is consistent with Theorem \ref{the:li}. 
	
	Figure \ref{fig:3} compares our approach with the four baseline algorithms on empirical risk when we set the privacy parameter $\epsilon$ to be $0.5$ and $1$, and $\delta$ to be $0.00001$. The results show that our approach has more stable update processing and has better utility than the other differentially private ADMM algorithms.

	\section{Related Work}   \label{sec:rela}
	
	
	\subsection{Distributed ADMM}
	ADMM demonstrates fast convergence in many applications and it is widely used to solve distributed optimizations. Shi et al. \citep{ShiLing14} focus the theoretical aspects on the convergence rate of distributed ADMM, and demonstrate that the iterates from distributed ADMM can converge to the optimal solution linearly under the assumptions that the objective function is strongly convex and Lipschitz smooth. Zhang and Kwok \citep{ZhangKw14} propose an asynchronous distributed ADMM by using partial barrier and bounded delay. Ling et al. \citep{LingShi15} design a linearized distributed ADMM where the augmented Lagrangian function is replaced by its first-order approximation to reduce the local computation cost. Song et al. \citep{SongYoon16} show that distributed ADMM can converge faster by adaptively choosing the penalty parameter. Makhdoumi and Ozdaglar \citep{MaOz17} demonstrate that the objective value with feasibility violation converges to the optimum at a rate of $O(1/t)$ by distributed ADMM, where $t$ is the number of iterations. 
	
	\subsection{Differentially Private Empirical Risk Minimization}
	There have been tremendous research efforts on differentially private empirical risk minimization \citep{ChMo11,BaSm14,WangYe17,ThSm13}. Chaudhuri et al. \citep{ChMo11} propose two perturbation methods: output perturbation and objective perturbation to guarantee $\epsilon$-differential privacy. Bassily et al. \citep{BaSm14} provide a systematic investigation of differentially private algorithms for convex empirical risk minimization and propose efficient algorithms with tighter error bound. Wang et al. \citep{WangYe17} focus on a more general problem: non-convex problem, and propose a faster algorithm based on a proximal stochastic gradient method. Smith and Thakurta \citep{ThSm13} explore the stability of model selection problems, and propose two differentially private algorithms based on perturbation stability and subsampling stability respectively.
	
	\subsection{Differentially Private ADMM-based Distributed Learning}
	Recently, there are some works focusing on differentially private ADMM-based distributed learning algorithms. Zhang and Zhu \citep{ZhZh17} propose two perturbation methods: primal perturbation and dual perturbation to guarantee dynamic differential privacy in ADMM-based distributed learning. Zhang et al. \citep{ZhKh18} propose to perturb the penalty parameter of ADMM to guarantee differential privacy. Huang et al. \citep{HuangHu19} propose an algorithm named DP-ADMM, where an approximate augmented Lagrangian function with time-varying Gaussian noise addition is adopted to update iterates while guaranteeing differential privacy, and theoretically analyze the convergence rate of their approach. Zhang et al. \citep{ZhangKh19} propose recycled ADMM with differential privacy guarantee where the results from odd iterations could be re-utilized by the even iterations, and thus half of updates incur no privacy leakage. Hu et al. \citep{HuLiu19} consider a setting where data features are distributed, and use ADMM with primal variable perturbation for distributed learning while guaranteeing differential privacy. Comparing with previous works on differentially private distributed ADMM, we propose an improved algorithm by performing multiple iterate updates with approximation per node in each iteration. Our theoretical analysis demonstrates the improvement of our approach.

	\section{Conclusion}  \label{sec:conc}
	In this paper, we have proposed a new differentially private distributed ADMM algorithm for a class of convex learning problems. In our approach, we have adopted the approximation when updating the primal variables and have allowed each node to perform such primal variable updates with differentially private noise for $l$ times in each iteration. We have analyzed the privacy guarantee of our proposed algorithm by properly setting the noise magnitude in Gaussian distribution and using the moments accountant method. We have theoretically analyzed the utility of our approach by the excess empirical risk with feasibility violation under the setting that the objective is Lipschitz and convex. Our theoretical results have shown that our approach can obtain higher accuracy if we set a larger $l$ and can achieve the error bounds, which are comparable to the state-of-art error bounds for differentially private empirical risk minimization.
	
		\section{Appendix}
	\subsection{Proof of Theorem \ref{the:li}} \label{app:a}
	In this appendix, we will give the proof of Theorem \ref{the:li}. Firstly, by assuming that the diameter of dual variable domain is bounded, namely $\Vert \boldsymbol{\gamma}_{i,j} \Vert \leq \beta$, we have:
	\begin{equation}
	\begin{split}
	& \mathbb{E}\bigg[\sum_{i \in \mathcal{D}}L_{\mathcal{D}_i}(\hat{\boldsymbol{w}}_i)-L_{\mathcal{D}_i}(\boldsymbol{w}^{*}) \bigg] + \beta \sum_{i \in \mathcal{V}}\sum_{j \in \mathcal{N}_i} \Vert \hat{\boldsymbol{w}}_i - \hat{\boldsymbol{w}}_j \Vert  \\
	= & \max_{\boldsymbol{\gamma}_{i,j}: \Vert \boldsymbol{\gamma}_{i,j} \Vert \leq \beta } \mathbb{E}\bigg[\sum_{i \in \mathcal{D}} \bigg( L_{\mathcal{D}_i}(\hat{\boldsymbol{w}}_i)-L_{\mathcal{D}_i}(\boldsymbol{w}^{*})  - \sum_{j \in \mathcal{N}_i} \langle \boldsymbol{\gamma}_{i,j} , \hat{\boldsymbol{w}}_i - \hat{\boldsymbol{w}}_j \rangle \bigg) \bigg] . \label{eq:ap1}
	\end{split}
	\end{equation}
	Due to the convexity of $L_{\mathcal{D}}(\cdot)$ and the definition of $\hat{\boldsymbol{w}}_i$, we have:
	\begin{equation}
	\begin{split}
	&L_{\mathcal{D}_i}(\hat{\boldsymbol{w}}_i)-L_{\mathcal{D}_i}(\boldsymbol{w}^{*}) - \sum_{j \in \mathcal{N}_i} \langle \boldsymbol{\gamma}_{i,j} , \hat{\boldsymbol{w}}_i - \hat{\boldsymbol{w}}_j \rangle\\
	\leq & \frac{1}{t} \sum_{k=1}^{t} \frac{1}{l} \sum_{r=0}^{l-1}  \big( L_{\mathcal{D}_i}(\tilde{\boldsymbol{w}}_i^{k,r})-L_{\mathcal{D}_i}(\boldsymbol{w}^{*}) - \sum_{j \in \mathcal{N}_i} \langle \boldsymbol{\gamma}_{i,j} , \tilde{\boldsymbol{w}}_i^{k,r} - \tilde{\boldsymbol{w}}_j^{k,r}  \rangle \big)  \\
	\leq & \frac{1}{t} \sum_{k=1}^{t}  \frac{1}{l} \sum_{r=0}^{l-1}  \big( \langle \nabla L_{\mathcal{D}_i}(\tilde{\boldsymbol{w}}_i^{k,r}), \tilde{\boldsymbol{w}}_i^{k,r} - \boldsymbol{w}^{*} \rangle  - \sum_{j \in \mathcal{N}_i} \langle \boldsymbol{\gamma}_{i,j} , \tilde{\boldsymbol{w}}_i^{k,r} - \tilde{\boldsymbol{w}}_j^{k,r}  \rangle \big). \label{eq:ap5}
	\end{split}
	\end{equation}
	Next, we analyze $\big\langle \nabla L_{\mathcal{D}_i}(\tilde{\boldsymbol{w}}_i^{k,r}), \tilde{\boldsymbol{w}}_i^{k,r} - \boldsymbol{w}^{*} \big\rangle$:
	\begin{equation}
	\begin{split}
	\big\langle \nabla  L_{\mathcal{D}_i}(\tilde{\boldsymbol{w}}_i^{k,r}), \tilde{\boldsymbol{w}}_i^{k,r} - & \boldsymbol{w}^{*} \big\rangle 
	=   \big\langle \nabla L_{\mathcal{D}_i}(\tilde{\boldsymbol{w}}_i^{k,r})+\boldsymbol{\xi}_i, \tilde{\boldsymbol{w}}_i^{k,r+1} - \boldsymbol{w}^{*} \big\rangle \\&  +  \big\langle \boldsymbol{\xi}_i, \boldsymbol{w}^* - \boldsymbol{w}_i^{k,r+1} \big\rangle + \big \langle \nabla L_{\mathcal{D}_i}(\tilde{\boldsymbol{w}}_i^{k,r})+\boldsymbol{\xi}_i,\tilde{\boldsymbol{w}}_i^{k,r} - \tilde{\boldsymbol{w}}_i^{k,r+1} \big\rangle .
	\end{split}
	\end{equation}
	If we define: $\boldsymbol{\xi}_i = \boldsymbol{\xi}_i^{k,r}/(2\rho\vert \mathcal{N}_i \vert+\eta_i^{k,r})$, according to the primal variable update \eqref{up:3} and \eqref{up:4}, we have:
	\begin{equation}
	\begin{split}
	& \big\langle \nabla L_{\mathcal{D}_i}(\tilde{\boldsymbol{w}}_i^{k,r})+\boldsymbol{\xi}_i, \tilde{\boldsymbol{w}}_i^{k,r+1} - \boldsymbol{w}^{*} \big\rangle \\
	= &\big\langle \nabla L_{\mathcal{D}_i}(\tilde{\boldsymbol{w}}_i^{k,r})- 2\boldsymbol{\gamma}_i^{k-1} + 2\rho \sum_{j \in \mathcal{N}_i}\big(\tilde{\boldsymbol{w}}_i^{k,r} -\frac{1}{2}(\tilde{\boldsymbol{w}}_i^{k-1}+\tilde{\boldsymbol{w}}_j^{k-1})\big) +\boldsymbol{\xi}_i,  \tilde{\boldsymbol{w}}_i^{k,r+1}  - \boldsymbol{w}^{*} \big\rangle \\& + 2 \big\langle \boldsymbol{\gamma}_i^{k-1} - \rho \sum_{j \in \mathcal{N}_i}\big(\tilde{\boldsymbol{w}}_i^{k,r} -\frac{1}{2}(\tilde{\boldsymbol{w}}_i^{k-1}+\tilde{\boldsymbol{w}}_j^{k-1})\big),\tilde{\boldsymbol{w}}_i^{k,r+1} - \boldsymbol{w}^{*}  \big\rangle\\
	= &   \big(\eta_i^{k,r+1}+2\rho\vert\mathcal{N}_i\vert \big)\big\langle \tilde{\boldsymbol{w}}_i^{k,r}-\tilde{\boldsymbol{w}}_i^{k,r+1},\tilde{\boldsymbol{w}}_i^{k,r+1} - \boldsymbol{w}^{*} \big\rangle\\
	&  +  2 \big\langle \boldsymbol{\gamma}_i^{k-1} - \rho \sum_{j \in \mathcal{N}_i}\big(\tilde{\boldsymbol{w}}_i^{k,r} -\frac{1}{2}(\tilde{\boldsymbol{w}}_i^{k-1}+\tilde{\boldsymbol{w}}_j^{k-1})\big),\tilde{\boldsymbol{w}}_i^{k,r+1} - \boldsymbol{w}^{*}  \big\rangle  .
	\end{split}
	\end{equation}
	We handle the two terms separately:
	\begin{equation}
	\begin{split}
	\big\langle \tilde{\boldsymbol{w}}_i^{k,r}-\tilde{\boldsymbol{w}}_i^{k,r+1},\tilde{\boldsymbol{w}}_i^{k,r+1} - \boldsymbol{w}^{*} \big\rangle 
	=  \frac{1}{2} \Vert \tilde{\boldsymbol{w}}_i^{k,r} -\boldsymbol{w}^* \Vert^2 - \frac{1}{2}\Vert \tilde{\boldsymbol{w}}_i^{k,r+1} -\boldsymbol{w}^* \Vert^2  - \frac{1}{2} \Vert \tilde{\boldsymbol{w}}_i^{k,r}-\tilde{\boldsymbol{w}}_i^{k,r+1} \Vert^2 ,
	\end{split}
	\end{equation}
	Based on the dual update \eqref{up:5} and the definition of $\tilde{\boldsymbol{w}}_i^{k}$, we have:
	\begin{equation}
	\begin{split}
	& \frac{1}{l} \sum_{r=0}^{l-1}  2 \big\langle \boldsymbol{\gamma}_i^{k-1} - \rho \sum_{j \in \mathcal{N}_i}\big(\tilde{\boldsymbol{w}}_i^{k,r} -\frac{1}{2}(\tilde{\boldsymbol{w}}_i^{k-1}+\tilde{\boldsymbol{w}}_j^{k-1})\big),\tilde{\boldsymbol{w}}_i^{k,r+1} - \boldsymbol{w}^{*}  \big\rangle \\
	=&\frac{1}{l} \sum_{r=0}^{l-1}  2\rho \sum_{j \in \mathcal{N}_i} \big\langle \tilde{\boldsymbol{w}}_i^{k,r+1}-\tilde{\boldsymbol{w}}_i^{k}, \boldsymbol{w}^* -\tilde{\boldsymbol{w}}_i^{k,r+1} \big\rangle  + 2 \sum_{j \in \mathcal{N}_i} \big\langle \boldsymbol{\gamma}_{i,j}^{k}, \tilde{\boldsymbol{w}}_i^{k} - \boldsymbol{w}^*  \big\rangle \\ & +\frac{1}{l} \sum_{r=0}^{l-1}  2\rho \sum_{j \in \mathcal{N}_i} \big\langle \tilde{\boldsymbol{w}}_i^{k,r}-\tilde{\boldsymbol{w}}_i^{k,r+1}, \boldsymbol{w}^* -\tilde{\boldsymbol{w}}_i^{k,r+1} \big\rangle\\
	&+  2\rho \sum_{j \in \mathcal{N}_i} \big\langle \frac{1}{2}(\tilde{\boldsymbol{w}}_i^{k}+\tilde{\boldsymbol{w}}_j^{k})-\frac{1}{2}(\tilde{\boldsymbol{w}}_i^{k-1}+\tilde{\boldsymbol{w}}_j^{k-1}),\boldsymbol{w}^* - \tilde{\boldsymbol{w}}_i^{k}\big\rangle .
	\end{split}
	\end{equation}
	Since we have:
	\begin{equation}
	\begin{split}
 	\sum_{r=0}^{l-1}  \big\langle  \tilde{\boldsymbol{w}}_i^{k,r+1}-\tilde{\boldsymbol{w}}_i^{k}, \boldsymbol{w}^* -\tilde{\boldsymbol{w}}_i^{k,r+1} \big\rangle 
	= \frac{1}{l}\sum_{r=0}^{l-1} \sum_{a=0}^{l-1}\big\langle \tilde{\boldsymbol{w}}_i^{k,r+1}-\tilde{\boldsymbol{w}}_i^{k,a+1}, \boldsymbol{w}^* -\tilde{\boldsymbol{w}}_i^{k,r+1} \big\rangle < 0 ,
	\end{split}
	\end{equation}
	and 
	\begin{equation}
	\begin{split}
	&\big\langle \frac{1}{2}(\tilde{\boldsymbol{w}}_i^{k}+\tilde{\boldsymbol{w}}_j^{k})-\frac{1}{2}(\tilde{\boldsymbol{w}}_i^{k-1}+\tilde{\boldsymbol{w}}_j^{k-1}),\boldsymbol{w}^* - \tilde{\boldsymbol{w}}_i^{k}\big\rangle \\
	\leq & \frac{1}{2}\big(\Vert \frac{1}{2}(\tilde{\boldsymbol{w}}_i^{k-1}+\tilde{\boldsymbol{w}}_j^{k-1})-\boldsymbol{w}^* \Vert^2-\Vert\frac{1}{2}(\tilde{\boldsymbol{w}}_i^{k}+\tilde{\boldsymbol{w}}_j^{k}) -\boldsymbol{w}^*\Vert^2  +\Vert\frac{1}{2}(\tilde{\boldsymbol{w}}_i^{k}-\tilde{\boldsymbol{w}}_j^{k})\Vert^2\big) \\
	= & \frac{1}{2}\big(\Vert \frac{1}{2}(\tilde{\boldsymbol{w}}_i^{k-1}+\tilde{\boldsymbol{w}}_j^{k-1})-\boldsymbol{w}^* \Vert^2-\Vert\frac{1}{2}(\tilde{\boldsymbol{w}}_i^{k}+\tilde{\boldsymbol{w}}_j^{k}) -\boldsymbol{w}^*\Vert^2 \big)+ \frac{1}{2\rho^2}\Vert \boldsymbol{\gamma}_{i,j}^{k-1}-\boldsymbol{\gamma}_{i,j}^{k} \Vert^2 ,
	\end{split}
	\end{equation}
	and by Young's inequality: 
	\begin{equation}
	\begin{split}
	\big \langle \nabla L_{\mathcal{D}_i}(\tilde{\boldsymbol{w}}_i^{k,r})+\boldsymbol{\xi}_i,\tilde{\boldsymbol{w}}_i^{k,r} - \tilde{\boldsymbol{w}}_i^{k,r+1} \big\rangle 
	\leq   \frac{1}{2\eta_i^{k,r+1}}\Vert  \nabla L_{\mathcal{D}_i}(\tilde{\boldsymbol{w}}_i^{k,r})+\boldsymbol{\xi}_i \Vert^2 + \frac{\eta_i^{k,r+1}}{2}\Vert \tilde{\boldsymbol{w}}_i^{k,r}-\tilde{\boldsymbol{w}}_i^{k,r+1} \Vert^2 ,
	\end{split}
	\end{equation}
	we can obtain:
	\begin{equation}
	\begin{split}
	&\frac{1}{l}\sum_{r=0}^{l-1}  \big\langle \nabla L_{\mathcal{D}_i}(\tilde{\boldsymbol{w}}_i^{k,r}), \tilde{\boldsymbol{w}}_i^{k,r} - \boldsymbol{w}^{*} \big\rangle  
	\\\leq & \frac{1}{l}\sum_{r=0}^{l-1} \big( \frac{\eta_i^{k,r+1}}{2}\big(\Vert \tilde{\boldsymbol{w}}_i^{k,r} -\boldsymbol{w}^* \Vert^2 +\Vert \tilde{\boldsymbol{w}}_i^{k,r+1} -\boldsymbol{w}^{*} \Vert^2 \big)   \\ & +   \frac{1}{2\eta_i^{k,r+1}}\Vert  \nabla L_{\mathcal{D}_i}(\tilde{\boldsymbol{w}}_i^{k,r})+\boldsymbol{\xi}_i \Vert^2 \big) + \frac{1}{l}\sum_{r=0}^{l-1} \big\langle \boldsymbol{\xi}_i, \boldsymbol{w}^* - \boldsymbol{w}_i^{k,r+1} \big\rangle \\ & + \rho\sum_{j \in \mathcal{N}_i}\big(\Vert\frac{1}{2}(\tilde{\boldsymbol{w}}_i^{k-1}+\tilde{\boldsymbol{w}}_j^{k-1})-\boldsymbol{w}^* \Vert^2-\Vert\frac{1}{2}(\tilde{\boldsymbol{w}}_i^{k}+\tilde{\boldsymbol{w}}_j^{k}) -\boldsymbol{w}^*\Vert^2\big) \\&  + \sum_{j \in \mathcal{N}_i}\frac{1}{\rho}\Vert \boldsymbol{\gamma}_{i,j}^{k-1}-\boldsymbol{\gamma}_{i,j}^{k} \Vert^2 + \frac{1}{l}\sum_{r=0}^{l-1}  \sum_{j \in \mathcal{N}_i} 2\big\langle \boldsymbol{\gamma}_{i,j}^{k}, \tilde{\boldsymbol{w}}_i^{k+1} - \boldsymbol{w}^*  \big\rangle . \label{eq:ap2}
	\end{split}
	\end{equation}
	Next, we analyze $ \sum_{i\in \mathcal{V}} \sum_{j \in \mathcal{N}_i} \langle - \boldsymbol{\gamma}_{i,j} , \tilde{\boldsymbol{w}}_i^{k} - \tilde{\boldsymbol{w}}_j^{k}  \rangle$:
	\begin{equation}
	\begin{split}
	\sum_{i\in \mathcal{V}} \sum_{j \in \mathcal{N}_i} & \langle - \boldsymbol{\gamma}_{i,j} , \tilde{\boldsymbol{w}}_i^{k} - \tilde{\boldsymbol{w}}_j^{k}  \rangle  
	= \sum_{i\in \mathcal{V}} \sum_{j \in \mathcal{N}_i} \big( \langle - \boldsymbol{\gamma}_{i,j}^{k}, \tilde{\boldsymbol{w}}_i^{k} - \boldsymbol{w}^*  \rangle \\& +\langle  \boldsymbol{\gamma}_{i,j}^{k}, \tilde{\boldsymbol{w}}_j^{k} - \boldsymbol{w}^*  \rangle + 2 \langle \boldsymbol{\gamma}_{i,j}^{k} - \boldsymbol{\gamma}_{i,j} , \frac{1}{2}(\tilde{\boldsymbol{w}}_i^{k}-\tilde{\boldsymbol{w}}_j^{k})\rangle \big) \\
	=  \sum_{i\in \mathcal{V}}  \sum_{j \in \mathcal{N}_i}  &\big( \langle - 2 \boldsymbol{\gamma}_{i,j}^{k}, \tilde{\boldsymbol{w}}_i^{k} - \boldsymbol{w}^*  \rangle  + 2 \langle \boldsymbol{\gamma}_{i,j}^{k} - \boldsymbol{\gamma}_{i,j} , \frac{1}{2}(\tilde{\boldsymbol{w}}_i^{k}-\tilde{\boldsymbol{w}}_j^{k})\rangle \big). \label{eq:ap3}
	\end{split}
	\end{equation}
	Furthermore, we have:
	\begin{equation}
	\begin{split}
	&\big\langle \boldsymbol{\gamma}_{i,j}^{k} - \boldsymbol{\gamma}_{i,j} , \frac{1}{2}(\tilde{\boldsymbol{w}}_i^{k}-\tilde{\boldsymbol{w}}_j^{k}) \big\rangle 
	= \frac{1}{\rho}\big\langle \boldsymbol{\gamma}_{i,j}^{k} - \boldsymbol{\gamma}_{i,j} , \boldsymbol{\gamma}_{i,j}^{k-1}-\boldsymbol{\gamma}_{i,j}^{k} \big\rangle \\
	= & \frac{1}{2\rho} \big( \Vert\boldsymbol{\gamma}_{i,j}^{k-1}-\boldsymbol{\gamma}_{i,j} \Vert^2 -\Vert\boldsymbol{\gamma}_{i,j}^{k}-\boldsymbol{\gamma}_{i,j} \Vert^2-\Vert \boldsymbol{\gamma}_{i,j}^{k-1}-\boldsymbol{\gamma}_{i,j}^{k} \Vert^2 \big). \label{eq:ap4}
	\end{split}
	\end{equation}
	Since we assume $L_\mathcal{D}(\cdot)$ is $c_2$-Lipschitz, we have:
	\begin{equation}
	\mathbb{E}\big[\Vert  \nabla L_{\mathcal{D}_i}(\tilde{\boldsymbol{w}}_i^{k,r})+\boldsymbol{\xi}_i \Vert^2\big] = \frac{c_2^2}{n^2}+ \frac{d c_0^2 c_1^2 t l 8\ln{(1.25/\delta)}}{\epsilon^2 m_i^2} .    
	\end{equation}
	Since we have $\mathbb{E}\big[\langle \boldsymbol{\xi}_i, \boldsymbol{w}^* - \boldsymbol{w}_i^{k,r} \rangle\big] = 0$, by assuming that the diameter of the $\mathcal{W}$ is bounded by $D$, and let $\eta_i^{k,r} = \frac{\sqrt{2 k r}}{D}\sqrt{\frac{c_2^2}{n^2}+ \frac{d c_0^2 c_1^2 t l 8\ln{(1.25/\delta)}}{\epsilon^2 m_i^2}} $, according to Eq. \eqref{eq:ap1}, Eq. \eqref{eq:ap5}, Eq. \eqref{eq:ap2}, Eq. \eqref{eq:ap3}, and Eq. \eqref{eq:ap4}, we can obtain:
	\begin{equation}
	\begin{split}
	& \mathbb{E}\bigg[\sum_{i \in \mathcal{D}}L_{\mathcal{D}_i}(\hat{\boldsymbol{w}}_i)-L_{\mathcal{D}_i}(\boldsymbol{w}^{*}) \bigg] + \beta \sum_{i \in \mathcal{V}}\sum_{j \in \mathcal{N}_i} \Vert \hat{\boldsymbol{w}}_i - \hat{\boldsymbol{w}}_j \Vert  \\ 
	\leq &  \sum_{i \in \mathcal{V}} \bigg( \frac{1}{t} \frac{1}{l} \sum_{k=1}^{t} \sum_{r=0}^{l-1} \mathbb{E} \big[\frac{\Vert  \nabla L_{\mathcal{D}_i}(\tilde{\boldsymbol{w}}_i^{k,r})+\boldsymbol{\xi}_i \Vert^2}{2\eta_i^{k,r+1}}\big]   + \frac{1}{t}\frac{1}{l}\frac{\eta_i^{t,l}}{2} D^2 \\& \quad \quad+ \frac{1}{t}\rho \vert \mathcal{N}_i\vert D^2   + \frac{1}{t}\frac{\vert \mathcal{N}_i\vert}{\rho}\max_{\boldsymbol{\gamma}_{i,j}: \Vert \boldsymbol{\gamma}_{i,j} \Vert \leq \beta } \Vert \boldsymbol{\gamma}_{i,j}^{0} -\boldsymbol{\gamma}_{i,j} \Vert^2 \bigg)\\
	\leq & \sum_{i \in \mathcal{V }}\bigg(\frac{\sqrt{2}D }{\sqrt{t \cdot l}}\big(\frac{c_2^2}{n^2}+ \frac{d c_0^2 c_1^2 t l 8\ln{(1.25/\delta)}}{\epsilon^2 m_i^2}\big)^{\frac{1}{2}}  + \frac{\rho \vert \mathcal{N}_i\vert D^2+\vert \mathcal{N}_i\vert\beta^2/\rho}{t}\bigg).
	\end{split}
	\end{equation}

	\bibliographystyle{jmlr}
	\bibliography{myrefer}

\end{document}